\documentclass[runningheads]{llncs}
\usepackage[T1]{fontenc}
\usepackage{graphicx}
\usepackage[ruled,linesnumbered, boxed]{algorithm2e}
\usepackage{enumerate}
\usepackage{amsmath}
\usepackage{booktabs}
\usepackage{multirow}
\usepackage[]{hyperref}
\usepackage{color} 
\newcommand{\xy}[1]{\textcolor{black}{#1}}
\newcommand{\zm}[1]{\textcolor{black}{#1}}
\usepackage{bbding}
\begin{document}

\title{Extracting Weighted Finite Automata from Recurrent Neural Networks for Natural Languages}

\titlerunning{Extracting WFA from RNNs for NL}
\author{Zeming Wei \and Xiyue Zhang \and Meng Sun\inst{(}\Envelope\inst{)}}
\authorrunning{Z. Wei et al.}
%
\institute{Peking University, Beijing 100871, China\\
\email{weizeming@stu.pku.edu.cn, \{zhangxiyue,sunm\}@pku.edu.cn}}

\maketitle              
\begin{abstract}
    Recurrent Neural Networks (RNNs) have achieved tremendous success in sequential data processing.
    However, it is quite challenging to interpret and verify RNNs' behaviors directly.
    To this end, 
    many efforts have been made
    to extract 
    finite automata from RNNs.
    Existing approaches such as exact learning are effective in extracting finite-state models to characterize the state dynamics of RNNs for formal languages, but are limited in the scalability to process natural languages. 
    Compositional approaches that are scablable to natural languages fall short in extraction precision.
    In this paper, 
    we identify the transition sparsity problem that heavily impacts the extraction precision.
    To address this problem, we propose a transition rule extraction approach, which is scalable to natural language processing models and effective in improving extraction precision. 
    Specifically, 
    we propose an empirical method to 
    complement
    the missing 
    rules
    in the transition diagram.
    In addition, we further adjust the transition matrices to
    enhance the context-aware ability of the extracted weighted finite automaton (WFA).
    Finally, we 
    propose two data augmentation tactics to track more dynamic behaviors of the target RNN.
    Experiments on two popular natural language datasets show that our method can extract WFA from RNN for natural language processing with better precision than existing approaches.
    Our code is available at \url{https://github.com/weizeming/Extract_WFA_from_RNN_for_NL}.
\end{abstract}
\section{Introduction}
    In the last decade, deep learning (DL) has been widely deployed
    in a range of applications, such as image processing~\cite{He_2016_CVPR}, speech recognition~\cite{abdel2014convolutional} and natural language processing~\cite{goldberg2017neural}.
    In particular, recurrent neural networks (RNNs) achieve great success in sequential data processing, e.g., time series forecasting~\cite{che2018recurrent}, text classification~\cite{wang2019convolutional} and language translation~\cite{datta2020neural}.
    However, the 
    complex internal design and gate control
    of RNNs make the interpretation and verification of their behaviors 
    rather challenging.
    To this end, much progress has been made to abstract RNN as a finite automaton,
    which is a finite state model with explicit states and transition matrix to
    characterize
    the behaviours of RNN in processing sequential data. 
    The extracted automaton also provides a practical foundation 
    for
    analyzing and verifying RNN behaviors,
    based on which existing mature techniques,
    such as logical formalism~\cite{logic} and model checking~\cite{modelChecking}, can be leveraged for RNN analysis.

   Up to the present, a series of extraction approaches leverage explicit learning algorithms (e.g., $L^*$ algorithm~\cite{angluin}) to extract a surrogate model of RNN.
    Such exact learning procedure has achieved great success in capturing the state dynamics of RNNs for processing formal languages~\cite{weiss2018,weiss2019,ok2020}.
    However, 
    the computational complexity of the exact learning algorithm limits its scalability to construct abstract models from RNNs for natural language tasks.
    
    Another technical line of automata extraction from RNNs is the compositional approach, which uses unsupervised learning algorithms to obtain discrete partitions of RNNs' state vectors and construct the transition diagram based on the discrete clusters and concrete state dynamics of RNNs.
    This approach 
    demonstrates better scalability and has been applied to robustness analysis and repairment of RNNs on large-scale tasks~\cite{wang2018,wang2018verification,du2019,dong2020,du2020,xie2021}.
    
    As a trade-off to the computational complexity, the compositional approach is faced with the problem of extraction consistency.
    Moreover, the alphabet size of natural language datasets is far larger than formal languages, but the extraction procedure is based on finite (sequential) data. 
    As a result, the transition dynamics \zm{are usually scarce when processing low-frequency tokens (words)}.
    
    However, the transition sparsity of the extracted automata for natural language tasks is yet to be addressed.
    In addition, state-of-the-art RNNs such as long short-term memory networks~\cite{lstm} show their great advantages on tracking long term 
    context dependency for natural language processing,
    \zm{but} the abstraction procedure inevitably leads to context information loss. 
    This motivates us to propose a heuristic method for transition rule adjustment to enhance the context-aware ability of the extracted model.

    In this paper, we propose an approach to extracting transition rules of weighted finite automata from RNNs for natural languages with a focus on the transition sparsity problem 
    and the loss of context dependency.
    Empirical investigation in Section \ref{sec:experiments} shows that the sparsity problem of transition diagram severely impacts the behavior consistency between RNN and the extracted model.
    To deal with the transition sparsity problem that no transition rules are learned at a certain state for a certain word, 
    which we refer to as \textit{missing rows} in transition matrices,
    we propose a novel method to fill in the transition rules for the missing rows based on the semantics of abstract states.
    Further, in order to enhance the context awareness ability of WFA, 
    we adjust the transition matrices to preserve part of the context information from the previous state, especially in the case of transition sparsity when the current transition rules cannot be relied on completely.
    Finally, we propose two tactics to augment the training samples to learn more transition behaviors of RNNs, which also alleviate the transition sparsity problem.
    Overall, our approach for transition rule extraction leads to better extraction consistency and can be applied to natural language tasks.

    To summarize, our main contributions are:
    \begin{enumerate}[(a)]
        \item A novel approach to extracting transition rules of WFA from RNNs to address the transition sparsity problem;
        \item An \zm{heuristic} method of adjusting transition rules to enhance the context-aware ability of WFA;
        \item A data augmentation method on training samples to track more transition behaviors of RNNs.
    \end{enumerate}

    The organization of this paper is as follows. 
    In Section \ref{sec:preliminary}, we show preliminaries about recurrent neural networks, weighted finite automata, and related notations and concepts.
    In Section \ref{sec:approach}, we present our transition rule extraction approach, including a generic outline on the automata extraction procedure, the transition rule complement approach for transition sparsity, the transition rule adjustment method for context-aware ability enhancement.
    We then present the data augmentation tactics in Section \ref{sec:augmentation} to reinforce the learning of dynamic behaviors from RNNs, along with the computational complexity analysis of the overall extraction approach.
    In Section \ref{sec:experiments}, we present the experimental evaluation towards the extraction consistency of our approach on two natural language tasks. 
    Finally, we discuss related works in Section \ref{sec:related} and conclude our work in Section \ref{sec:conclusion}.

\section{Preliminaries}
\label{sec:preliminary}
In this section,
we present the notations and definitions that will be used throughout the paper.

Given a finite alphabet $\Sigma$, we 
use $\Sigma^*$ to denote the set of sequences over $\Sigma$ and $\varepsilon$ to denote the empty sequence.
For $w\in\Sigma^*$, we use $|w|$ to denote its length, its $i$-th word as $w_i$ and its prefix with length $i$ as $w[:i]$.
For $x\in\Sigma$, $w\cdot x$ represents the concatenation of $w$ and $x$. 

\begin{definition}[RNN]
A \textit{Recurrent Neural Network (RNN)} for natural language processing is a tuple
$\mathcal{R}=(\mathcal X, \mathcal{S}, \mathcal{O},f, p)$,
where $\mathcal X$ is the input 
space; 
$\mathcal{S}$ is the internal state space;
$\mathcal{O}$ is the probabilistic output space;
$f: \mathcal{S}\times\mathcal X\to \mathcal{S}$ is the transition function;
$p: \mathcal{S}\to\mathcal{O}$ is the prediction function.
\end{definition}

\paragraph{RNN Configuration.}
In this paper, we consider RNN as a black-box model and focus on its stepwise probabilistic output for each input sequence.
The following definition of configuration characterizes the probabilistic outputs in response to a sequential input
fed to RNN.
    Given an alphabet $\Sigma$,
    let $\xi:\Sigma\to\mathcal X$ be the function that maps each word in $\Sigma$
    to its embedding vector in $\mathcal X$.
    We define $f^*: \mathcal{S}\times\Sigma^*\to\mathcal{S}$ recursively as $f^*(s_0,\xi(w\cdot x))=f(f^*(s_0, \xi(w)),\xi(x))$ and $f^*(s_0,\varepsilon)=s_0$,
    where $s_0$ is the initial state of $\mathcal{R}$.
The RNN configuration $\delta:\Sigma^*\to \mathcal{O}$ is defined as $\delta(w)=p(f^*(s_0, w))$.

\paragraph{Output Trace.}
To record the stepwise behavior of RNN when processing an input sequence $w$, 
we define the \textit{Output Trace} of $w$, i.e., the probabilistic output sequence, as
$T(w)=\{\delta(w[:i])\}_{i=1}^{|w|}$.
The $i$-th item of $T(w)$ indicates the probabilistic output given by $\mathcal{R}$ 
after taking the prefix of $w$ with length $i$ as input.

\begin{definition}[WFA]
Given a finite alphabet $\Sigma$, a \textit{Weighted Finite Automaton (WFA)} over $\Sigma$ is a tuple
$\mathcal{A}=(\hat S, \Sigma, E, \hat s _0, I, F)$, where $\hat S$ is the finite 
set of abstract states;
$E=\{E_\sigma|\sigma\in \Sigma\}$ is the
set of transition matrix $E_\sigma$ with size $|\hat S|\times|\hat S|$ for each token $\sigma\in\Sigma$;
$\hat s _0\in \hat S$ is the initial state;
$I$ is the initial vector, a row vector with size $|\hat S|$;
$F$ is the final vector, a column vector with size $|\hat S|$.
\end{definition}

\paragraph{Abstract States.}
Given a RNN $\mathcal{R}$ and a dataset $\mathcal D$,
let $\hat{\mathcal O}$ denote all stepwise probabilistic outputs given by executing $\mathcal R$ on $\mathcal D$,
i.e. $\hat{\mathcal{O}}=\bigcup\limits_{w\in\mathcal D}T(w)$.
The abstraction function $\lambda:\hat{\mathcal{O}}\to \hat{S}$ maps each probabilistic output to an abstract state $\hat s\in \hat S$.
As a result, the output set is divided into a number of abstract states by $\lambda$.
For each $\hat s\in\hat S$, the state $\hat s$ has explicit semantics that
the probabilistic outputs corresponding to $\hat s$ has similar distribution.
In this paper, we leverage the \textit{k-means} algorithm to construct the abstraction function.
We cluster all probabilistic outputs in $\hat O$ into some abstract states.
In this way, we construct the set of abstract states $\hat S$ with these discrete clusters and an initial state $\hat s_0$.

For a state $\hat s \in \hat S$, 
we define the \textit{center} of $\hat s$ as the 
average value of the probabilistic outputs $\hat o\in \hat O$ which are mapped to $\hat s$. 
More formally, the center of $\hat s$ is defined as follows:
$$\rho(\hat s)=\underset{\lambda({\hat o}) = \hat s}{\mathrm{Avg}}\{\hat o\}.$$
The center $\rho(\hat s)$ represents an approximation for the distribution tendency of probabilistic outputs $\hat o$ in $\hat s$.
For each state $\hat s \in \hat S$, we use the center $\rho(\hat s)$ as its weight,
as the center captures an approximation of the distribution tendency of this state.
Therefore, the final vector $F$ is $(\rho(\hat s_0),\rho(\hat s_1),\cdots, \rho(\hat s_{|\hat S|-1}))^t$.

\paragraph{Abstract Transitions.} In order to capture the dynamic behavior of RNN $\mathcal{R}$,
we define the abstract transition as a triple $(\hat{s}, \sigma, \hat{s}')$ where the original state $\hat{s}$
is the abstract state corresponding to a specific output $y$, i.e. $\hat s=\lambda(y)$; $\sigma$ is the next word of the input sequence to consume; $\hat s '$
is the destination state $\lambda(y')$ after $\mathcal{R}$ reads $\sigma$ and outputs $y'$. 
We use $\mathcal T$ to denote the set of all abstract transitions tracked from the execution of $\mathcal R$ on training samples.

\paragraph{Abstract Transition Count Matrices.} For each word $\sigma\in \Sigma$,
the abstract transition count matrix
of $\sigma$ is a matrix $\hat T_\sigma$
with size $|\hat S|\times|\hat S|$. 
The count matrices records the number of times that each abstract transition triggered.
Given the set of abstract transitions $\mathcal{T}$, the count matrix of $\sigma$ can be calculated as
$$\hat T_\sigma[i,j]=\mathcal{T}.count((\hat s_i,\sigma,\hat s_j)),\quad 1\le i,j\le |\hat S|.$$ 


As for the remaining components,
the alphabet $\Sigma$ is consistent with the alphabet of training set $\mathcal D$.
The initial vector $I$ is formulated according to the initial state $\hat s_0$.

For an input sequence $w=w_1w_2\cdots w_n\in\Sigma^*$, the WFA will calculate
its weight following $$I\cdot E_{w_1}\cdot E_{w_2}\cdots E_{w_n}\cdot P.$$


\begin{figure}[t]
    \includegraphics[width=\textwidth]{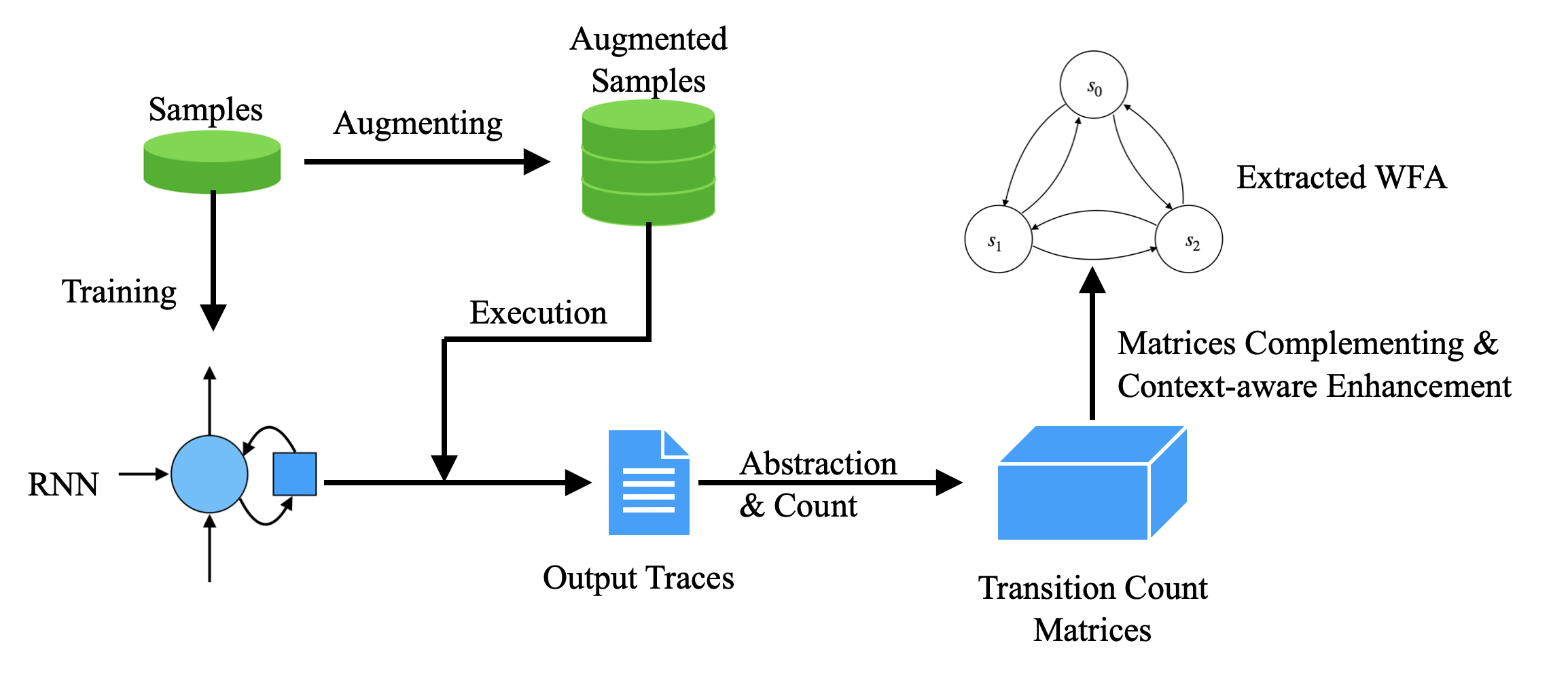}
    \caption{An illustration of our approach to extracting WFA from RNN.} \label{fig1}
\end{figure}

\section{Weighted Automata Extraction Scheme}
\label{sec:approach}
\subsection{Outline}
\label{subsec:outline}


We present the workflow of our extraction procedure
in Fig.~\ref{fig1}.
As the first step,
we generate augmented sample set ${\mathcal D}$ from the original training set $\mathcal D_0$ to enrich the transition dynamics of RNN behaviors and alleviate the transition sparsity.
Then, we execute RNN $\mathcal{R}$ on the augmented sample set ${\mathcal D}$,
and record the probabilistic output trace $T(w)$ of each input sentence $w\in \mathcal D$.
With the output set $\hat O=\bigcup\limits_{w\in \mathcal D}T(w)$, we cluster the probabilistic outputs into abstract states $\hat S$,
and generate abstract transitions $\mathcal T$ from the output traces $\{T(w)|w\in \mathcal D\}$.
All transitions constitute the abstract transition count matrices $\hat T_\sigma$ for all $\sigma\in\Sigma$.

Next, we construct the transition matrices $E=\{E_\sigma|\sigma\in\Sigma\}$.
    Based on the abstract states $\hat S$ and count matrices $\hat T$,
    we construct the transition matrix $E_\sigma$ for each word $\sigma\in\Sigma$.
    Specifically, we use frequencies to calculate the transition probabilities.
    Suppose that there are $n$ abstract states in $\hat S$.
    The $i$-th row of $E_\sigma$, which indicates the  probabilistic transition distribution over states when $\mathcal{R}$ is in state $\hat s _i$
    and 
    consumes $\sigma$, is calculated as 
    \begin{equation}
        E_\sigma[i,j] = \frac{\hat T_\sigma[i,j]}{\sum\limits_{k=1}^n \hat T_\sigma[i,k]}.\label{transition matrix}
    \end{equation}
    This empirical rule faces the problem that the denominator of (\ref{transition matrix}) could be zero,
    which means that the word $\sigma$ never appears when the RNN $\mathcal{R}$ is in abstract state $\hat s_i$.
    In this case, one should decide how to fill 
    in the transition rule of the \textit{missing rows} in $E_\sigma$.
    In Section \ref{subsec:missing_rows},
    we present a novel approach for transition rule complement.
    Further, 
    to preserve more contextual information 
    in the case of transition sparsity,
    we propose an approach to enhancing the context-aware ability of WFA by adjusting the transition matrices,
    which is discussed in Section \ref{subsec:context}.
\xy{Note that our approach is generic and could be applied to other RNNs besides the domain of natural language processing.
}
\subsection{Missing Rows Complementing}
\label{subsec:missing_rows}
    Existing approaches for transition rule extraction usually face the problem of transition sparsity,
    i.e., \textit{missing rows} in the transition diagram.
    In the context of formal languages, the probability of occurrence of missing rows is quite limited,
    since the size of the alphabet is small and each token in the alphabet can appear sufficient number of times.
    However, in the context of natural language processing, the occurrence of missing rows is quite frequent.
    The following proposition gives an approximation of 
    the occurrence frequency
    of missing rows.

    \begin{proposition}
        Assume 
        an alphabet $\Sigma$ with $m=|\Sigma|$ words, a natural language dataset $\mathcal D$ over $\Sigma$ which has $N$ words in total,
        a RNN $\mathcal{R}$ trained on $\mathcal D$,
        the extracted abstract states $\hat S$ and 
        transitions 
        $\mathcal{T}$.
        Let $\sigma_i$ 
        denote 
        the $i$-th most frequent word 
        occurred
        in $\mathcal D$ and $t_i = \mathcal{T}.count((*,\sigma_i,*))$
        indicates
        the occurrence times of $\sigma_i$ in $\mathcal D$.
        The median of $\{t_i|1\le i\le m\}$ can be estimated as $$t_{[\frac{m}{2}]}=\frac{2N}{m\cdot \ln m}.$$
    \end{proposition}

    \begin{proof}
        The Zipf's law~\cite{powers1998applications} shows that $$\frac{t_i}{N}\approx \frac{i^{-1}}{\sum\limits_{k=1}^m k^{-1}}.$$
        Note that $\sum\limits_{k=1}^m k^{-1}\approx \ln m$ and take $i$ to be $\frac{m}{2}$, we complete our proof.
    
    \end{proof}
    \begin{example}
        In the QC news dataset~\cite{qc}, which has $m=20317$ words in its alphabet and $N=205927$ words in total, 
        the median of $\{t_i\}$ is approximated to $\frac{2N}{m\cdot \ln m}\approx 2$.
        This indicates that about half of $E_\sigma$ are constructed with no more than $2$ transitions.
        In practice,
        the number of abstract states is usually far more than the transition numbers for these words,
        making most of rows 
        of their transition matrices \textit{missing rows}.
    \end{example}

    Filling the missing row with $\vec 0 $ is a simple solution, since no information were provided from the transitions.
    However, as estimated above, this solution will 
    lead to the problem of transition sparsity, i.e.,
    the transition matrices for uncommon words are nearly null.
    Consequently,
    if the input sequence includes some uncommon words,
    the weights over states tend to vanish.
    We refer to this solution as \textit{null filling}.
    
    Another simple idea is to use the uniform distribution over states for fairness.
    In \cite{weiss2019}, the
    uniform distribution is used as the transition distribution
    for unseen tokens
    in the context of formal language tasks.
    However, for natural language processing,
    this solution still loses information of the current word, despite that it avoids the weight vanishment over states.
    We refer to this solution as \textit{uniform filling}.
    \cite{zhang2021} uses the \textit{synonym} transition distribution for an unseen token at a certain state. However, it increases the computation overhead when performing inference on test data, since it requires to calculate and sort the distance between the available tokens at a certain state and the unseen token.
    
    \zm{To this end, we propose a novel approach to constructing the transition matrices based on two empirical observations.
    First, each abstract state $\hat s \in \hat S$ has explicit semantics, i.e. the probabilistic distribution over labels,
    and similar abstract states tend to share more similar transition behaviours. 
    The similarity of abstract states is defined by their semantic distance as follows.
    }
    

    \begin{definition}[State Distance]
        For two abstract states $\hat s_1$ and $\hat s_2$, the distance between $\hat s_1$ and $\hat s_2 $ is
        defined by the Euclidean distance between their center: $$dist(\hat s_1, \hat s_2) = \lVert \rho(\hat s_1)-\rho(\hat s_2) \rVert_2.$$
    \end{definition}

    We calculate the distance between 
    each pair of abstract states,
    which forms a \textit{distance matrix} $M$ 
    where each element 
    $M[i,j] = dist(\hat s_i,\hat s_j)$ for $1\le i,j\le |\hat S|$.
    For a missing row in $E_\sigma$, following the heuristics that 
    similar abstract states are more likely to have similar behaviors,
    we observe the transition behaviors 
    from other abstract states,
    and simulate the missing transition behaviors
    weighted by distance \zm{between states}.
    \zm{Particularly}, in order to avoid numerical underflow, we leverage \textit{softmin} on distance to 
    bias the weight to states that share more similarity.
    Formally,
    for a missing row $E_\sigma[i]$, the weight of information 
    set for
    another row $E_\sigma[j]$ is defined by $e^{-M[i,j]}$.

    \zm{Second, it is also observed that sometimes the RNN just remains in the current state after reading a certain word.
    Intuitively, this is because part of words in the sentence do not deliver significant information in the task.
    Therefore, we consider simulating
    behaviors from other states whilst remaining 
    in the current state with a certain probability.}
    
    
    \zm{In order to balance the trade-off between referring to behaviors from other states and remaining still,}
    we introduce a hyper-parameter $\beta$ named \textit{reference rate}, such that when WFA 
    is faced with a missing row,
    it has a probability of $\beta$ to refer 
    to the transition behaviors
    from other states,
    and 
     in the meanwhile
    has a probability of $1-\beta$ to keep still.
    \xy{We select the parameter $\beta$ according to the proportion of self-transitions, i.e., transitions $(\hat s, \sigma, \hat s')$ in $\mathcal T$ where $\hat s=\hat s'$.}

    To sum up, the complete 
    transition rule
    for the missing row is 
    \begin{equation}
    E_\sigma[i,j] = \beta\cdot \frac{\sum\limits_{k=1}^n e^{- M[i,k]} \cdot\hat T _\sigma[k,j]}{\sum\limits_{l=1}^n\sum\limits_{k=1}^n e^{- M[i,k]} \cdot\hat T _\sigma[k,l]} + (1-\beta)\cdot\delta_{i,j}.\label{missing row}
    \end{equation}
    Here $\delta_{i,j}$ is the Kronecker symbol:
    $$\delta_{i,j} =\begin{cases}1, & j=i\\0, & j\ne i\end{cases}.$$
    In practice, we can calculate $\sum\limits_{k=1}^n e^{- M[i,k]} \cdot\hat T _\sigma[k,j]$ for each $j$ and 
    then make division on their summation once and for all,
    which can reduce the computation overhead on transition rule extraction. 

\subsection{Context-Aware Enhancement}
\label{subsec:context}
    
    For NLP tasks, the memorization of long-term context information is crucial.
    One of the advantages of RNN and its advanced design LSTM networks is 
    the ability to capture long-term dependency.
    We 
    expect the extracted WFA to simulate the step-wise behaviors of RNNs whilst 
    keeping track of context information  along with the state transition.
    To this end, we propose an approach to adjusting the transition matrix
    such that
    the WFA 
    can remain in the current state with a certain probability.

    Specifically, we select a hyper-parameter $\alpha\in [0,1]$ as the \textit{static probability}.
    For each word $\sigma\in\Sigma$ and its transition matrix $E_\sigma$,
    we replace the matrix with the \textit{context-aware enhanced matrix} $\hat E _\sigma$ 
    as follows:
    \begin{equation}
        \hat E _\sigma = \alpha\cdot I_n + (1-\alpha)\cdot E_\sigma
        \label{context}
    \end{equation}
    where $I_n$ is the identity matrix. 
    
    The context-aware enhanced matrix has explicit semantics.
    When the WFA is in state $\hat s_i$ and 
    ready to process a new word $\sigma$,
    it has a probability of $\alpha$ (the \textit{static probability}) to remain in $\hat s_i$,
    or follows the original transition distribution $E_\sigma[i,j]$ with a probability $1-\alpha$.

    Here we present an illustration of how context-aware enhanced matrices deliver long-term context information.
    Suppose that a context-aware enhanced WFA $\mathcal A$ is 
    processing a sentence $w\in \Sigma^*$ with length $|w|$.
    We denote $d_i$ as the distribution over all abstract states after $\mathcal A$ reads the prefix $w[:i]$, 
    and
    particularly $d_0=I$ is the initial vector of $\mathcal A$.
    We use $Z_i$ to denote 
    the decision made by $\mathcal A$ 
    based on $d_{i-1}$ and the original transition matrix $E_{w_i}$.
    Formally, $d_i = d_{i-1}\cdot \hat E_{w_i}$ and $Z_i = d_{i-1}\cdot E_{w_i}$.

    The $d_i$ can be regarded as the information 
    obtained from the prefix $w[:i]$ by $\mathcal A$ before it 
    consumes $w_{i+1}$,
    and $Z_i$ can be considered as the decision made by $\mathcal A$ after it reads $w_{i}$. 
    \begin{proposition}
        The $i$-th step-wise information $d_i$ which delivered by
        processing
        $w[:i]$ contains the decision information $Z_j$
        of prefix $w[:j]$ with a proportion of $(1-\alpha)\cdot \alpha^{i-j}$, $1\le j\le i$.
    \end{proposition}
    \begin{proof}
    Since $\hat E_{w_i} = \alpha\cdot I_n + (1-\alpha)\cdot E_{w_i}$, we can calculate that
    \begin{equation}
        d_i = d_{i-1} \cdot \hat E_{w_i} = d_{i-1}\cdot [\alpha\cdot I_n + (1-\alpha)\cdot E_{w_i}]=\alpha\cdot d_{i-1} + (1-\alpha)\cdot Z_i.\label{recursion}
    \end{equation}
    Using (\ref{recursion}) recursively, we have $$d_i = (1-\alpha)\sum_{k=1}^{i} \alpha^{i-k}\cdot Z_k + \alpha^i \cdot I.$$

    \end{proof}

    This shows the information delivered by $w[:i]$ 
    refers to the decision made by $\mathcal A$ on each prefix included in $w[:i]$,
    and the portion vanishes exponentially.
    The effectiveness of the
    context-aware transition matrix adjustment method will be discussed in Section \ref{sec:experiments}.

    The following example presents 
    the complete approach for transition rule extraction, i.e.,  
    to generate the transition matrix $\hat E_\sigma$ with the missing row filled in and context enhanced,
    from the count matrix $\hat T_\sigma$ for a word $\sigma\in\Sigma$.
    \begin{example}
    
    Assume that there are three abstract states in $\hat S = \{\hat s_1,\hat s_2,\hat s_3\}.$
        Suppose the count matrix for $\sigma$ is $\hat T_\sigma$.
        $$
        \hat T_\sigma = \begin{bmatrix} 1 & 3 & 0 \\ 1 & 1 & 0 \\ 0 & 0 & 0\end{bmatrix},
        E_\sigma = \begin{bmatrix}0.25 & 0.75 & 0\\ 0.5 & 0.5 & 0 \\ 0.15 & 0.35 & 0.5\end{bmatrix},
        \hat E_\sigma = \begin{bmatrix}0.4 & 0.6 & 0 \\ 0.4 & 0.6 & 0 \\ 0.12 & 0.28 & 0.6 \end{bmatrix}.
        $$
        For the first two rows (states), there exist transitions for $\sigma$,
        thus we can calculate the transition distribution of
        these two rows in $E_\sigma$ 
        in the usual way.
        However, the third row is a \textit{missing row}. 
        We set the \textit{reference rate} as $\beta=0.5$,
        and suppose that the distance between states satisfies $e^{-M[1,3]}= 2 e^{-M[2,3]}$, generally indicating
        the distance between $\hat s_1$ and $\hat s_3$ is nearer than $\hat s_2$ and $\hat s_3$.
        With the transitions from $\hat s_1$ and $\hat s_2$, we can complement the transition rule of the third row in $E_\sigma$ through (\ref{missing row}).
        The result shows that the behavior from $\hat s_3$ is more similar to $\hat s_1$ than $\hat s_2$, due to the 
        nearer distance.
        Finally, we construct $\hat E_\sigma$ with $E_\sigma$. Here we take the \textit{static probability} $\alpha = 0.2$,
        thus $\hat E_\sigma = 0.2\cdot I_3 + 0.8\cdot E_\sigma$. The result shows that the WFA with $\hat E_\sigma$ has higher probability
        to remain in the current state after consuming $\sigma$, which can preserve more information from the prefix before $\sigma$.

    \end{example}

\section{Data Augmentation}
\label{sec:augmentation}
Our proposed approach for transition rule extraction provides a solution to the transition sparsity problem.
Still, we hope to learn more dynamic transition behaviors from the target RNN, 
especially for the words with relatively low frequency to characterize their transition dynamics sufficiently based on the finite data samples.
Different from formal languages, we can generate more 
natural language samples
automatically,
as long as 
the augmented sequential data are sensible with clear semantics and compatible with the original learning task.
Based on the augmented samples, we are able to track more behaviors of the RNN 
and build the abstract model with higher precision.
In this section, we 
introduce two data augmentation tactics 
for natural language processing tasks: \textit{Synonym Replacement}
and \textit{Dropout}.

\subsubsection{Synonym Replacement.} 
Based on the distance quantization among the word embedding vectors,
we can 
obtain a list of synonyms for each word in $\Sigma$.
For a word $\sigma\in\Sigma$, the \textit{synonyms} of $w$ are 
defined as
the 
top $k$ most similar words of $\sigma$ in $\Sigma$, where $k$ is a hyper-parameter. 
The similarity 
among the words 
is 
calculated
based on the Euclidean distance between the word embedding vectors over $\Sigma$.

Given a dataset $\mathcal D_0$ over $\Sigma$, for each sentence $w\in \mathcal D_0$,
we generate a new sentence $w'$ by replacing some words in $w$ with their synonyms.
Note that we hope that the uncommon words in $\Sigma$ should appear more times,
so as to gather more dynamic behaviors of RNNs when processing such words.
Therefore, we set the probability that a certain word $\sigma\in w$ gets replaced to be in a 
negative correlation
to its frequency of occurrence, i.e. the $i$-th most frequent word is replaced with a probability $\frac{1}{i+1}$.

\subsubsection{Dropout.}
Inspired by the regularization
 technique \textit{dropout},
we also propose a similar tactic to generate new sentences from $\mathcal D_0$.
Initially, 
we 
introduce a new word named \textit{unknown word} and denote it as $\left\langle \textbf{unk}\right\rangle$.
For the sentence in $w\in\mathcal D_0$ that has been processed by synonym replacing,
we further replace the words that hasn't been replaced with $\left\langle \textbf{unk}\right\rangle$ 
with a certain probability.
Finally, new sentences generated by both synonym replacement and dropout form the augmented dataset $\mathcal D$.

With the dropout tactic, we can observe the behaviors of RNNs when it 
processes an unknown word $\hat\sigma\not\in \Sigma$
that hasn't appeared in $\mathcal D_0$. 
Therefore, the extracted WFA can also show better generalization ability.
An example of 
generating a new sentence from $\mathcal D_0$ is shown as follows.

\begin{example}
    Consider a
    sentence $w$ from the original training set $\mathcal D_0$,
    $w=$[``I'', ``really'', ``like'', ``this'', ``movie'']. First, the word ``like'' is chosen to be replaced by one of its synonym ``appreciate''.
    Next, the word ``really'' is dropped from the sentence, i.e. replaced by the unknown word $\langle \textbf{unk}\rangle$.
    Finally, we get a new sentence $w'=$[``I'', ``$\langle \textbf{unk}\rangle$'', ``appreciate'', ``this'', ``movie''] and 
    put it into the
    augmented dataset $\mathcal D$.

    Since the word ``appreciate'' may be an uncommon word in $\Sigma$, we can capture a new transition information for it given by RNNs. Besides, we can also observe the behavior of RNN when it reads an unknown word after the prefix [``I''].
\end{example}

\subsubsection{Computational Complexity.}
    The time complexity of the whole workflow is analyzed as follows.
    Suppose that the set of training samples $\mathcal D_0$ 
    has
    $N$ words in total
    and its alphabet $\Sigma$ contains $n$ words, and is augmented as $\mathcal D$ with $t$ epochs (i.e. each sentence in $\mathcal{D}_0$ is transformed to $t$ new sentences in $\mathcal{D}$),
    hence $|\mathcal{D}|=(t+1)N$.
    Assume that a probabilistic output of RNNs is a $m$-dim vector,
    and the abstract states set $\hat S$ contains $k$ states.
    
    To start with, 
    the augmentation of $\mathcal D_0$ and tracking of probabilistic outputs in $\mathcal{D}$ will
    be completed in $\mathcal O(|\mathcal D|)=\mathcal O(t\cdot N)$ time.
    Besides, the time complexity of k-means clustering algorithm is $\mathcal O(k\cdot |\mathcal D|)=\mathcal O(k\cdot t\cdot N)$.
    The count of abstract transitions will be done in $\mathcal O(n)$.
    As for the processing of transition matrices, we need
    to
    calculate the transition 
    probability for each 
    word $\sigma$ with each 
    source state $\hat s_i$ and 
    destination state $\hat s_j$,
    which costs $\mathcal O(k^2\cdot n)$ time. Finally, the context-aware enhancement on transition matrices takes $\mathcal{O}(k\cdot n)$ time.

    Note that $\mathcal O(n)=\mathcal O(N)$, hence we can conclude that the time complexity of our whole workflow
    is $\mathcal O(k^2\cdot t\cdot N)$.
    So the time complexity of our approaches only takes linear time w.r.t. the 
    size of the dataset,
    which provides theoretical 
    extraction overhead for
    large-scale data applications.
\section{Experiments}
\label{sec:experiments}
In this section, we 
evaluate our extraction approach on two 
natural language
datasets and demonstrate its performance 
in terms of precision and scalability.
\paragraph{Datasets and RNNs models.} 
    We select two popular datasets for NLP tasks and train the target RNNs on them.
    \begin{enumerate}
        \item The CogComp QC Dataset (abbrev. QC)~\cite{qc} contains news titles which are labeled with different topics.
        The dataset is divided into 
        a training set containing 20k samples and 
        a test set containing 8k samples.
        Each sample is labeled with one of seven categories.
        We train an LSTM-based RNN model $\mathcal R$ on 
        the training set, which achieves an accuracy of $81\%$
        on the test set.

        \item The Jigsaw Toxic Comment Dataset (abbrev. Toxic)~\cite{toxic} contains comments from Wikipedia's talk page edits,
        with each comment labeled 
        toxic or not. 
        We select 25k non-toxic samples and toxic samples respectively,
        and divide them into the training set and test set in a ratio of four to one. 
        We train another LSTM-based RNN model which achieves $90\%$ accuracy.
    \end{enumerate}

    \paragraph{Metrics.} For the purpose of 
    representing the behaviors of RNNs better,
    we 
    use \textit{Consistency Rate (CR)} as our 
    evaluation metric.
    For a sentence in the test set $w\in \mathcal{D}_{test}$,
    we denote $\mathcal R(w)$ and $\mathcal A(w)$ as the 
    prediction
    results 
    of the RNNs and WFA, respectively.
    The \textit{Consistency Rate} 
    is defined formally as $$CR = \frac{|\{w\in D_{test}:\mathcal A (w) = \mathcal R (w)\}|}{|\mathcal D_{test}|}.$$

    \subsubsection{Missing Rows Complementing.}
    As discussed in Section \ref{subsec:missing_rows},
    we take two approaches as baselines, the \textit{null filling}
    and the \textit{uniform filling}. 
    The two WFA extracted with these two approaches are denoted as $\mathcal A_0$
    and $\mathcal A_U$, respectively. 
    The WFA extracted by our \textit{empirical filling} approach is denoted as $\mathcal A_E$. 

    \begin{table}
        \centering
        \caption{Evaluation results of different filling approaches on missing rows.}\label{missing row cr}
        \renewcommand\arraystretch{1.5}

        \setlength{\tabcolsep}{3mm}{
        \begin{tabular}{|c|c|c|c|c|c|c|}
            \hline
            \multirow{2}{*}{Dataset} &
            
            \multicolumn{2}{|c|}{$\mathcal{A}_0$} & \multicolumn{2}{c|}{$\mathcal{A}_U$} & \multicolumn{2}{c|}{$\mathcal{A}_E$}\\
            \cline{2-7}
             & CR(\%) & Time(s) & CR(\%) & Time(s) & CR(\%) & Time(s)\\
            \hline
             QC & 26 & 47 & 60 & 56 & \textbf{80} & 70\\
            \hline
            Toxic & 57 & 167 & 86 & 180 & \textbf{91} & 200 \\
            \hline
            \end{tabular}}
    \end{table}

    Table \ref{missing row cr} shows the evaluation results of three rule filling approaches.
    We conduct the comparison experiments on QC and Toxic datasets 
    and select the cluster number for state abstraction 
    as $40$ and $20$ for the QC and Toxic datasets, respectively.
    
    The three columns labeled with the type of WFA show the 
    evaluation results of different approaches. 
    For the $\mathcal A_0$ based on blank filling, the WFA returns the weight of most sentences in $\mathcal D$ with $\vec 0$,
    which fails to provide sufficient information for prediction. 
    For the QC dataset, only a quarter of sentences in the test set are classified correctly.
    The second column shows that the performance of $\mathcal A_U$ is better than $\mathcal A_0$.
    The last column presents the evaluation result of $\mathcal A_E$, which fills the missing rows by our approach.
    In this experiment, the hyper-parameter \textit{reference rate} is set as $\beta=0.3$.
    We can see that our empirical approach achieves significantly better accuracy, which is $20\%$ and $5\%$ higher than uniform filling on the two datasets, respectively.

    The columns labeled \textit{Time} show the execution time of the whole extraction workflow, from tracking transitions to evaluation on test set, but not include the training time of RNNs.
We can see that the extraction overhead of our approach ($\mathcal A_E$) is
about the same as $\mathcal{A}_U$ and $\mathcal A_0$.

    \subsubsection{Context-Aware Enhancement.}
    In this experiment, we leverage the context-aware enhanced matrices when constructing the WFA.
    We adopt 
    the same configuration on cluster numbers $n$ from the comparison experiments above, i.e.
    $n=40$ and $n=20$. The columns titled \textit{Configuration} 
    indicate
    if the extracted WFA leverage context-aware matrices.
    We also take the WFA with different filling approaches, the uniform filling and empirical filling, into comparison.
    \xy{Experiments on null filling is omitted due to limited precision.}
    
    \begin{table}
        \centering
        \caption{Evaluation results of with and without context-aware enhancement.}\label{context cr}
        \renewcommand\arraystretch{1.5}

        \setlength{\tabcolsep}{4mm}{
        \begin{tabular}{|c|c|c|c|c|c|}
            \hline
            \multirow{2}{*}{Dataset}& \multirow{2}{*}{Configuration} & \multicolumn{2}{c|}{$\mathcal{A}_U$} & \multicolumn{2}{c|}{$\mathcal{A}_E$}\\
            \cline{3-6}
            & & CR(\%) & Time(s) & CR(\%) & Time(s) \\
            \hline
            \multirow{2}{*}{QC} 
            & None & 60 & 56 & 80 & 70\\
            \cline{2-6}
            & Context & 71 & 64 & \textbf{82} & 78\\
            \hline
            \multirow{2}{*}{Toxic} 
            & None & 86 & 180 & 91 & 200\\
            \cline{2-6}
            & Context & 89 & 191 & \textbf{92} & 211\\
            \hline
            \end{tabular}}
    \end{table}
    
    The experiment results
    are in Table \ref{context cr}.
    For the QC dataset, we 
    set the \textit{static probability} as $\alpha=0.4$. 
    The consistency rate of WFA $\mathcal A_U$ 
    improves 
    $11\%$
    with the context-aware enhancement,
    and $\mathcal A_E$ improves $2\%$. As for the Toxic dataset, we take $\alpha=0.2$
    and the consistency rate of the two WFA improves $3\%$ and $1\%$ respectively.
    This shows that the WFA with context-aware enhancement remains more information from the prefixes of sentences,
    making it simulate RNNs better.

    Still, the context-aware enhancement processing costs little time, since we only calculate the adjusting formula (\ref{context}) for each
    $E_\sigma$ in $E$. 
    The additional
    extra time consumption is 8s for the QC dataset and 11s for the Toxic dataset.

    \subsubsection{Data Augmentation}
    Finally, we evaluate the WFA extracted with transition behaviors from augmented data.
    Note that the two experiments above are based on the primitive training set $\mathcal D_0$.
    In this experiment, we leverage the data augmentation tactics to generate the augmented training set $\mathcal D$,
    and extract WFA with data samples from $\mathcal D$.
    In order to get best performance, we build WFA with contextual-aware enhanced matrices.

    \begin{table}
        \centering
        \caption{Evaluation results of with and without data augmentation.}\label{aug cr}
        \renewcommand\arraystretch{1.5}

        \setlength{\tabcolsep}{4mm}{
        \begin{tabular}{|c|c|c|c|c|c|}
            \hline
            \multirow{2}{*}{Dataset}& \multirow{2}{*}{Samples} & \multicolumn{2}{c|}{$\mathcal{A}_U$} & \multicolumn{2}{c|}{$\mathcal{A}_E$}\\
            \cline{3-6}
            & & CR(\%) & Time(s) & CR(\%) & Time(s) \\
            \hline
            \multirow{2}{*}{QC} 
            & $\mathcal D_0$ & 71 & 64 & 82 & 68\\
            \cline{2-6}
            & $\mathcal D$ & 76 & 81 & \textbf{84} & 85\\
            \hline
            \multirow{2}{*}{Toxic} 
            & $\mathcal D_0$ & 89 & 191 & 92 & 211\\
            \cline{2-6}
            & $\mathcal D$ & 91 & 295 & \textbf{94} & 315\\
            \hline
            \end{tabular}}
    \end{table}

    Table \ref{aug cr} shows the results of consistency rate of WFA extracted with and without augmented data.
    The rows labeled $\mathcal D_0$ show the results of WFA that are extracted with the primitive training set, 
    and the result from the augmented data is shown in rows labeled $\mathcal D$.
    With more transition behaviors tracked, the WFA extracted with $\mathcal D$ 
    demonstrates better precision.
    Specifically, the WFA extracted with
    both empirical filling and context-aware enhancement achieves a further $2\%$ increase in consistency rate
    on the two datasets.

    To summarize, by using our transition rule extraction approach, 
    the consistency rate of extracted WFA on the QC dataset and the Toxic dataset achieves $84\%$ and $94\%$, respectively. 
    Taking the primitive extraction algorithm with uniform filling as baseline, 
    of which experimental results 
    in terms of CR 
    are $60\%$ and $86\%$,
    our 
    approach achieves an improvement of $22\%$ and $8\%$ in consistency rate.
    As for the time complexity, the time consumption of our approach 
    increases 
    from $56s$ to $81s$ on QC dataset,
    and from $180s$ to $315s$ on Toxic dataset, 
    which indicates the efficiency and scalability 
    of our rule extraction approach. 
    There is no significant time cost of 
    adopting our approach further for complicated natural language tasks.
    We can conclude that our 
    transition rule extraction approach 
    makes better approximation of RNNs,
    and 
    is also efficient enough 
    to be applied to practical applications for large-scale 
    natural language tasks. 
\section{Related Work} 
\label{sec:related}
Many research efforts
have been 
made to abstract, verify and repair RNNs.
As Jacobsson reviewed in~\cite{jacobsson2005}, the rule extraction approach of RNNs can be divided into two categories:
pedagogical approaches and compositional approaches.

\paragraph*{Pedagogical Approaches.}
Much progress has been achieved by using pedagogical approaches to abstracting RNNs by leveraging explicit learning algorithms,
such as the $L^*$ algorithm~\cite{angluin}.
The earlier work dates back to two decades ago, when Omlin et al. attempted to extract
a finite model from Boolean-output RNNs~\cite{omlin1992,omlin1996,omlin1996b}.
Recently, Weiss et al. proposed to levergae the $L^*$ algorithm to extract DFA from RNN-acceptors~\cite{weiss2018}.
Later, they 
presented a weighted extension of $L^*$ algorithm that extracted probabilistic determininstic finite automata (PDFA)
from RNNs~\cite{weiss2019}.
Besides, Okudono et al. proposed another weighted extension of $L^*$ algorithm to extract WFA from real-value-output RNNs~\cite{ok2020}.

The pedagogical approaches have achieved great success in abstracting RNNs for small-scale languages,
particularly formal languages. 
Such exact learning approaches have intrinsic limitation in 
the scalability of the language complexity,
hence they are not suitable for automata extraction from natural language processing models.

\paragraph*{Compositional Approach.}
Another technical line
is the compositional approach, 
which generally leverages
unsupervised algorithms (e.g. k-means, GMM) to cluster state vectors as abstract states~\cite{zeng,cechin}.
Wang et al. studied the key factors that
influence the reliability of extraction process, and proposed an empirical rule to extract DFA from RNNs~\cite{wang2018}.
Later, Zhang et al. followed the state encoding of compositional approach and proposed a WFA extraction approach
from RNNs~\cite{zhang2021},
which can be applied to both grammatical languages and natural languages.
In this paper, our proposal of extracting WFA from RNNs also falls into the line of compositional approach,
but aims at proposing transition rule extraction method to address the transition sparsity problem 
and enhance the context-aware ability.

Recently, many of the verification, analysis and repairment works also leverage similar approaches to abstract RNNs as a more explicit model,
such as \textit{DeepSteller}~\cite{du2019}, \textit{Marble}~\cite{du2020} and \textit{RNNsRepair}~\cite{xie2021}.
These works achieve great progress in analyzing and repairing RNNs, but have strict requirements of scalability to large-scale tasks,
particularly natural language processing.
The proposed approach, which 
demonstrates
better precision and scalability, shows great potential for further applications such as RNN analysis and Network repairment.
We consider applying our method to 
RNN analysis as future work.


\section{Conclusion}
\label{sec:conclusion}
This paper presents a novel 
approach 
to extracting transition rules of
weighted finite automata from recurrent neural networks.
We measure the distance between abstract states and complement the transition rules of \textit{missing rows}.
In addition, we present an \zm{heuristic}
method 
to enhance the context-aware ability of the extracted WFA.
We further propose two augmentation tactics to track more transition behaviours of RNNs.
Experiments on two natural language datasets show that the WFA extracted with our approach
achieve better consistency with target RNNs.
The theoretical estimation of computation complexity and experimental results demonstrate
that our rule extraction approach
can be applied to 
natural language datasets and complete the extraction procedure efficiently for large-scale tasks.

%
%
\subsubsection{Acknowledgements}
This research was sponsored by the National Natural Science Foundation of China under Grant No. 62172019, 61772038,
and CCF-Huawei Formal Verification Innovation Research Plan.

%
%
%
\bibliographystyle{splncs04}
\bibliography{reference.bib}

\begin{thebibliography}{10}
\providecommand{\url}[1]{\texttt{#1}}
\providecommand{\urlprefix}{URL }
\providecommand{\doi}[1]{https://doi.org/#1}

\bibitem{abdel2014convolutional}
Abdel-Hamid, O., Mohamed, A.r., Jiang, H., Deng, L., Penn, G., Yu, D.:
  Convolutional neural networks for speech recognition. IEEE/ACM Transactions
  on audio, speech, and language processing  \textbf{22}(10),  1533--1545
  (2014)

\bibitem{angluin}
Angluin, D.: Learning regular sets from queries and counterexamples.
  Information and computation  \textbf{75}(2),  87--106 (1987)

\bibitem{modelChecking}
Baier, C., Katoen, J.P.: Principles of model checking. MIT press (2008)

\bibitem{cechin}
Cechin, A.L., Regina, D., Simon, P., Stertz, K.: State automata extraction from
  recurrent neural nets using k-means and fuzzy clustering. In: 23rd
  International Conference of the Chilean Computer Science Society, 2003. SCCC
  2003. Proceedings. pp. 73--78. IEEE (2003)

\bibitem{che2018recurrent}
Che, Z., Purushotham, S., Cho, K., Sontag, D., Liu, Y.: Recurrent neural
  networks for multivariate time series with missing values. Scientific reports
   \textbf{8}(1),  1--12 (2018)

\bibitem{datta2020neural}
Datta, D., David, P.E., Mittal, D., Jain, A.: Neural machine translation using
  recurrent neural network. International Journal of Engineering and Advanced
  Technology  \textbf{9}(4),  1395--1400 (2020)

\bibitem{dong2020}
Dong, G., Wang, J., Sun, J., Zhang, Y., Wang, X., Dai, T., Dong, J.S., Wang,
  X.: Towards interpreting recurrent neural networks through probabilistic
  abstraction. In: 2020 35th IEEE/ACM International Conference on Automated
  Software Engineering (ASE). pp. 499--510. IEEE (2020)

\bibitem{du2020}
Du, X., Li, Y., Xie, X., Ma, L., Liu, Y., Zhao, J.: Marble: Model-based
  robustness analysis of stateful deep learning systems. In: Proceedings of the
  35th IEEE/ACM International Conference on Automated Software Engineering. pp.
  423--435 (2020)

\bibitem{du2019}
Du, X., Xie, X., Li, Y., Ma, L., Liu, Y., Zhao, J.: Deepstellar: Model-based
  quantitative analysis of stateful deep learning systems. In: Proceedings of
  the 2019 27th ACM Joint Meeting on European Software Engineering Conference
  and Symposium on the Foundations of Software Engineering. pp. 477--487 (2019)

\bibitem{logic}
Gastin, P., Monmege, B.: A unifying survey on weighted logics and weighted
  automata. Soft Computing  \textbf{22}(4),  1047--1065 (2018)

\bibitem{goldberg2017neural}
Goldberg, Y.: Neural network methods for natural language processing. Synthesis
  lectures on human language technologies  \textbf{10}(1),  1--309 (2017)

\bibitem{He_2016_CVPR}
He, K., Zhang, X., Ren, S., Sun, J.: Deep residual learning for image
  recognition. In: Proceedings of the IEEE Conference on Computer Vision and
  Pattern Recognition (CVPR) (June 2016)

\bibitem{lstm}
Hochreiter, S., Schmidhuber, J.: Long short-term memory. Neural computation
  \textbf{9}(8),  1735--1780 (1997)

\bibitem{jacobsson2005}
Jacobsson, H.: Rule extraction from recurrent neural networks: Ataxonomy and
  review. Neural Computation  \textbf{17}(6),  1223--1263 (2005)

\bibitem{toxic}
Jigsaw: Toxic comment classification challenge,
  \url{https://www.kaggle.com/c/jigsaw-toxic-comment-classification-challenge}
  Accessed April 16, 2022

\bibitem{qc}
Li, X., Roth, D.: Learning question classifiers. In: COLING 2002: The 19th
  International Conference on Computational Linguistics (2002)

\bibitem{ok2020}
Okudono, T., Waga, M., Sekiyama, T., Hasuo, I.: Weighted automata extraction
  from recurrent neural networks via regression on state spaces. In:
  Proceedings of the AAAI Conference on Artificial Intelligence. vol.~34, pp.
  5306--5314 (2020)

\bibitem{omlin1996}
Omlin, C.W., Giles, C.L.: Extraction of rules from discrete-time recurrent
  neural networks. Neural networks  \textbf{9}(1),  41--52 (1996)

\bibitem{omlin1996b}
Omlin, C.W., Giles, C.L.: Rule revision with recurrent neural networks. IEEE
  Transactions on Knowledge and Data Engineering  \textbf{8}(1),  183--188
  (1996)

\bibitem{omlin1992}
Omlin, C., Giles, C., Miller, C.: Heuristics for the extraction of rules from
  discrete-time recurrent neural networks. In: [Proceedings 1992] IJCNN
  International Joint Conference on Neural Networks. vol.~1, pp. 33--38. IEEE
  (1992)

\bibitem{powers1998applications}
Powers, D.M.: Applications and explanations of zipf’s law. In: New methods in
  language processing and computational natural language learning (1998)

\bibitem{wang2018verification}
Wang, Q., Zhang, K., Liu, X., Giles, C.L.: Verification of recurrent neural
  networks through rule extraction. arXiv preprint arXiv:1811.06029  (2018)

\bibitem{wang2018}
Wang, Q., Zhang, K., Ororbia~II, A.G., Xing, X., Liu, X., Giles, C.L.: An
  empirical evaluation of rule extraction from recurrent neural networks.
  Neural computation  \textbf{30}(9),  2568--2591 (2018)

\bibitem{wang2019convolutional}
Wang, R., Li, Z., Cao, J., Chen, T., Wang, L.: Convolutional recurrent neural
  networks for text classification. In: 2019 International Joint Conference on
  Neural Networks (IJCNN). pp.~1--6. IEEE (2019)

\bibitem{weiss2018}
Weiss, G., Goldberg, Y., Yahav, E.: Extracting automata from recurrent neural
  networks using queries and counterexamples. In: International Conference on
  Machine Learning. pp. 5247--5256. PMLR (2018)

\bibitem{weiss2019}
Weiss, G., Goldberg, Y., Yahav, E.: Learning deterministic weighted automata
  with queries and counterexamples. In: Wallach, H., Larochelle, H.,
  Beygelzimer, A., d\textquotesingle Alch\'{e}-Buc, F., Fox, E., Garnett, R.
  (eds.) Advances in Neural Information Processing Systems. vol.~32. Curran
  Associates, Inc. (2019),
  \url{https://proceedings.neurips.cc/paper/2019/file/d3f93e7766e8e1b7ef66dfdd9a8be93b-Paper.pdf}

\bibitem{xie2021}
Xie, X., Guo, W., Ma, L., Le, W., Wang, J., Zhou, L., Liu, Y., Xing, X.:
  Rnnrepair: Automatic rnn repair via model-based analysis. In: International
  Conference on Machine Learning. pp. 11383--11392. PMLR (2021)

\bibitem{zeng}
Zeng, Z., Goodman, R.M., Smyth, P.: Learning finite state machines with
  self-clustering recurrent networks. Neural Computation  \textbf{5}(6),
  976--990 (1993)

\bibitem{zhang2021}
Zhang, X., Du, X., Xie, X., Ma, L., Liu, Y., Sun, M.: Decision-guided weighted
  automata extraction from recurrent neural networks. In: Thirty-Fifth AAAI
  Conference on Artificial Intelligence (AAAI). pp. 11699--11707. AAAI Press
  (2021)

\end{thebibliography}

\end{document}